\documentclass{article}
\usepackage[utf8]{inputenc}
\usepackage[OT1]{fontenc} 

\usepackage{graphicx} %
\usepackage{xcolor}
\usepackage{wrapfig}
\usepackage{float}
\usepackage{subcaption}
\usepackage{booktabs}
\usepackage[font=small]{caption} %
\usepackage{enumitem}
\usepackage{tablefootnote}

\PassOptionsToPackage{sort, numbers,square}{natbib}
\usepackage{amsmath,amssymb,amsthm}
\usepackage{color,mathtools}
\usepackage[round]{natbib}

\usepackage{multirow}
\usepackage{enumitem}

\usepackage{algorithm}
\usepackage{algpseudocode}

\usepackage{bm}

\usepackage[hyperindex,breaklinks]{hyperref}

\usepackage{amssymb,amsmath,amsthm,dsfont}

\providecommand{\lin}[1]{\ensuremath{\left\langle #1 \right\rangle}}

  \providecommand{\R}{\mathbb{R}} 
  
  \providecommand{\E}{{\mathbb E}}
  
  \DeclareMathOperator*{\argmin}{arg\,min}

  \providecommand{\0}{\mathbf{0}}
  \providecommand{\1}{\mathbf{1}}

  \providecommand{\cc}{\mathbf{c}}

  \let\ggg\gg
  \renewcommand{\gg}{\mathbf{g}}

  \let\lll\ll
  \renewcommand{\ll}{\mathbf{l}}

  \providecommand{\xx}{\mathbf{x}}
  \providecommand{\yy}{\mathbf{y}}

  \providecommand{\cD}{\mathcal{D}}

  \providecommand{\cO}{\mathcal{O}}

\RequirePackage[colorinlistoftodos,bordercolor=orange,backgroundcolor=orange!20,linecolor=orange,textsize=scriptsize]{todonotes}
\providecommand{\mycomment}[2]{\todo[caption={}]{\textbf{#1: }#2}}%
\providecommand{\inlinecomment}[3]{%
{\color{#1}#2: #3}}%
\newcommand\commenter[2]%
{%
  \expandafter\newcommand\csname i#1\endcsname[1]{\inlinecomment{#2}{#1}{##1}}
  \expandafter\newcommand\csname #1\endcsname[1]{\mycomment{#1}{##1}}
}

\newtheorem{lemma}{Lemma}
\newtheorem{corollary}[lemma]{Corollary}
\newtheorem{remark}{Remark}
\newtheorem{assumption}{Assumption}
\newtheorem{theorem}[lemma]{Theorem}

\DeclarePairedDelimiterX{\inp}[2]{\langle}{\rangle}{#1, #2}
\DeclarePairedDelimiterX{\abs}[1]{\lvert}{\rvert}{#1}
\DeclarePairedDelimiterX{\norm}[1]{\lVert}{\rVert}{#1}
\DeclarePairedDelimiterX{\cbr}[1]{\{}{\}}{#1} 
\DeclarePairedDelimiterX{\rbr}[1]{(}{)}{#1} 
\DeclarePairedDelimiterX{\sbr}[1]{[}{]}{#1} 

\definecolor{mydarkblue}{rgb}{0,0.08,0.45}
\hypersetup{ %
    colorlinks=true,
    linkcolor=mydarkblue,
    citecolor=mydarkblue,
    filecolor=mydarkblue,
    }

\def\papertitle{Characterizing \& Finding Good Data Orderings for Fast Convergence of Sequential Gradient Methods}

\usepackage[ textheight=9in,
    textwidth=5.5in,
    top=1in]{geometry}

\makeatletter
\renewcommand{\maketitle}{%
  \vbox{%
    \hsize\textwidth
    \linewidth\hsize
    \centering
    {\LARGE\bf \@title\par}
   
      \def\And{%
        \end{tabular}\hfil\linebreak[0]\hfil%
        \begin{tabular}[t]{c}\bf\rule{\z@}{24\p@}\ignorespaces%
      }
      \begin{tabular}[t]{c}\bf\rule{\z@}{24\p@}\@author\end{tabular}%
      \@thanks
  }
}
\makeatother

\title{\papertitle}
\author{Amirkeivan Mohtashami\\
EPFL\\
\And
Sebastian U. Stich\\
CISPA\footnotemark\\
\And
Martin Jaggi\\
EPFL\\
}
\date{}

\commenter{seb}{red}
\commenter{keivan}{blue}

\usetikzlibrary{calc}

\begin{document}
	
\maketitle\footnotetext{CISPA Helmholtz Center for Information Security}

\begin{abstract}
	While SGD, which samples from the data with replacement is widely studied in theory, a variant called Random Reshuffling (RR) is more common in practice. RR iterates through random permutations of the dataset and has been shown to converge faster than SGD.
	When the order is chosen deterministically, a variant called incremental gradient descent (IG), the existing convergence bounds show improvement over SGD but are worse than RR. However, these bounds do not differentiate between a good and a bad ordering and hold for the worst choice of order.  Meanwhile, in some cases, choosing the right order when using IG can lead to convergence faster than RR. In this work, we  quantify the effect of order on convergence speed, obtaining convergence bounds based on the chosen sequence of permutations while also recovering previous results for RR.  In addition, we show benefits of using structured shuffling when various levels of abstractions (e.g.\ tasks, classes, augmentations, etc.) exists in the dataset in theory and in practice. Finally, relying on our measure, we develop a greedy algorithm for choosing good orders during training, achieving superior performance (by more than 14 percent in accuracy) over RR. %
\end{abstract}

\section{Introduction}
Variants of Gradient Descent are widely used for optimization of machine learning models over a dataset \cite{bottou2010large,sun2020optimization}. This can be modeled as the finite-sum minimization problem
\begin{equation}
	\label{eq:main_opt_problem}
	f^\star := \min_{\xx \in  \R^{d}} \left[ f(\xx) := \frac{1}{N}\sum_{i=1}^{N} f_i(\xx) \right] \, .
\end{equation}

In practice, computing a full gradient is expensive and therefore the gradient of individual functions $f_i$ (or a mini-batch of them) is used. While it is possible to sample this function uniformly at random, it has been observed \cite{bottou2012stochastic,feng2012towards,recht2013parallel} that traversing a (possibly random) permutation of the  functions works better in practice.  Recent theoretical works confirmed this observation, showing that using a random permutation instead of random sampling can lead to faster convergence \cite{mishchenko2020random, nguyen_unified_2021, safran_how_2020,rajput_permutation-based_2021,yun_minibatch_2021}.

In this variant of gradient descent, called   \textbf{Random Reshuffling (RR)}, a new random permutation is chosen each time one pass over the functions is completed. However, the improved convergence holds even when using the same permutation during training instead of choosing a new random one each epoch \cite{mishchenko2020random,nguyen_unified_2021}. This approach is called \textbf{single shuffling} when the permutation is chosen at random, or \textbf{incremental gradient descent (IG)} when the permutation is chosen deterministically. Still, the rates obtained for these two variants are worse than the one obtained for random reshuffling, the gap depending on the number of functions $N$. However, these bounds hold even for the worst choice of orderings. Meanwhile, examples exists where choosing the right order can lead to faster convergence than even Random Reshuffling \cite{rajput_closing_2020}. Therefore, a convergence bound depending on the choice of order is needed and missing from literature. More importantly, it is currently not clear how to compare two orders in terms of their effect on convergence speed which can be useful for designing an order selection algorithm.

In this work, we address this gap by introducing a measure for quantifying the suitability of each order of functions for the next epoch. We obtain upper bounds on convergence in the general non-convex setting, yielding convergence bounds tailored to any sequence of dataset orderings, while recovering and unifying the previous bounds on IG and RR.  In addition, we allow access to the gradient of individual functions through a noisy oracle. To understand why this is important, note that $f_i$ can correspond to an individual data point but also may correspond to larger entities such as a worker in federated learning settings \cite{mcmahan2017communication,Kairouz2019:federated}. In this case, even computing the gradient of individual functions is expensive. By allowing a noisy oracle for the gradient, we facilitate using an approximate version of the gradient, for example by computing it over a single mini-batch instead of the whole part of dataset available to the worker.  

In addition to providing rigorous theoretical bounds, our measure also is useful in practice since it can be used in order to select orders more intelligently than randomly picking a permutation. To demonstrate this, by looking for orders that minimize our measure, we design a heuristic algorithm that can be used for choosing the right order in the next epoch. We showcase the effectiveness of our algorithm in practice by using it for training a neural network on an image classification task, namely CIFAR10 dataset. We show that our algorithm can even outperform random reshuffling in certain scenarios, improving the accuracy by more than 14 percent.

Choosing orders deterministically or uniformly at random are two extremes of the spectrum. A middle ground is randomly selecting the permutation from a more limited set of permutation.
While we allow a noisy oracle for $\nabla f_i$ in general, in certain scenarios this noise might be more structured. Motivated by federated learning settings \cite{Kairouz2019:federated}, where $f_i$ corresponds to $i$-th worker which itself is the average of loss for data points in the $i$-th worker, we consider the case when each $f_i$ is itself sum of a finite set of functions. We show that in this case, shuffling the top-level functions ($f_i$) followed by shuffling the low-level functions, a structure we refer to as \textbf{two-level shuffling},  can be beneficial both in theory and in practice. 

In summary, our contributions are:
\begin{itemize}[leftmargin=12pt,nosep]
	\item Quantifying the effect of a chosen order on convergence, providing a unified framework that yields order specific convergence bounds while unifying previous bounds on IG and RR.
	\item Proposing a greedy method for finding a good order during training and establishing its usefulness in practical settings of deep network training, proving the existence of practically feasible algorithms that can beat random-reshuffling.
	\item Showing the possibility of intelligent random order selection by proving the benefits of choosing the permutation randomly from a limited set of permutations having a certain structure 
	both in theory and in practice.
\end{itemize}

\section{Related Works}
Several previous work considered the effect of order on training. \citet{shah2020choosing} suggests selecting a sample with the smallest loss from a set of data points sampled with replacement to make the training more robust. On the other end of the spectrum, \citet{shumailov2021manipulating} show that an adversary can slow down the training by choosing a bad order of data. At the final stages of writing this paper, we discovered a simultaneous work \cite{lu2022a} that uses similar techniques to obtain convergence bounds depending on the order of examples and find good orders in practice. However, they do not consider a noisy access to the gradient. Moreover, we further consider the effect of imposing a structure over the random permutation on the convergence. 

Several recent literature also focus on obtaining both upper and lower bounds for permutation-based SGD. For strongly convex functions with smooth components having Lipschitz Hessian,  \citet{gurbuzbalaban_why_2021} show a $\cO(\frac{1}{T^2})$ bound, where $T$ is the number of epochs, but requires additional assumptions such as boundedness of parameters during training. Under the same conditions, \citet{haochen_random_2019} obtain the bound $\cO(\frac{1}{N^2T^2})$.  In a more general settings without requiring Lipschitz Hessian, \citet{mishchenko2020random} and \citet{nguyen_unified_2021} obtain the bound $\cO(\frac{1}{NT^2})$. Additionally, for the smooth non-convex case considered in our work, the prove the bound $\cO(\frac{1}{N^{1/3}T^{2/3}})$ . The bound for IG is only investigated in the worst-case order selection, leading to an inferior bound $\cO(\frac{1}{T^{2/3}})$ \cite{mishchenko2020random}.

In \cite{safran_how_2020}, a $\Omega(\frac{1}{N^2T^2})$ lower bound is established for random-reshuffling. On the other hand, \citet{rajput_permutation-based_2021} shows that for 1-dimensional functions with smooth Hessian, a good order exists which yields exponential convergence, beating RR. However, their proof is non-constructive and does not provide a technique for obtaining this order. When the number of dimensions are allowed to grow more than $1$, 
 \citet{rajput_permutation-based_2021} show that any permutation-based SGD is lower bounded by $\Omega(\frac{1}{N^3T^2})$ while even in 1-dimensional a lower bound $\Omega(\frac{1}{T^2})$ holds when the components are not convex. 

Combinations of variance-reduced variants of SGD and random reshuffling have also been investigated in the literature \cite{shamir_without-replacement_2016}. Extending these results to order-specific convergence bounds is grounds for future work. 

Importance sampling \cite{alain2015variance, katharopoulos2018not,stich2017safe}, uses only a subset of the dataset in each epoch to save on computation. In curriculum learning \cite{bengio_curriculum_learning_2009}, samples are divided into easy and hard classes and only easy samples are used at the beginning in order to allow the model to learn better. Here, we mostly focus our discussion on the case where the whole dataset is traversed completely in one epoch. However, our framework allows passing through a subset of the dataset as long as certain assumptions hold and can be used to analyze training methods that rely on using a subset of the dataset or in choosing a good subset to use.

\section{Order-Dependent Convergence Bound}
\subsection{Setup}
We look into optimizing the sum of $N$ functions, i.e.\ the optimization problem described in (\ref{eq:main_opt_problem}). 

We assume the individual functions are $L$-smooth.

\begin{assumption}[$L$-smoothness]
	\label{ass:lsmooth}
	Each function $f_i \colon \R^d \to \R$ is differentiable and there exists a constant $L > 0$  such that:\vspace{-1mm}
	\begin{align}
		\label{eq:lsmooth}
		\norm{\nabla f_i(\xx)- \nabla f_i(\yy)} &\leq L \norm{\xx - \yy}\,, \ \ \  \forall \xx,\yy \in \R^{d} \,.
	\end{align}
\end{assumption}

When considering how this formulation corresponds to optimizing a model over a dataset, a common setting is to let $f_1, \ldots, f_N$ correspond to individual data points. However, they can also represent higher abstraction levels. For example, consider an image classification task where random rotations of data points are added to the dataset for augmentation purposes which is a widely used technique \cite{shorten2019survey}. In this case, $f_i$ can be the expected values of the loss over various rotations of a single image. On a higher abstraction level, $f_i$ can correspond to the average of loss of all data points belonging to $i$-th class. Alternatively, $f_i$ can correspond to different workers in federated learning settings. 

However, in these cases, usually computing the exact value of $f_i$ is expensive as it requires iterating over several data points. Therefore, an approximation of the function is used for computing the gradient. For example, only a single rotation of the image is used to compute the loss or the loss for one of the data points available in the worker is returned. In order to  allow this behavior in our framework, we assume we can only access each function's gradient $\nabla f_i(\xx)$ through a noisy oracle $\nabla F_i(\xx, \xi) := \nabla f_i(\xx) + \xi$ where $\xi$ is picked randomly according to a distribution $\cD_i(\xx)$. Similar assumption is widely used to analyze gradient descent with sampling with replacement (SGD) where a noisy oracle for the global function $f$ is assumed \cite{Arjevani2018:delayed,Bottou2018:book}. However, when considering permutation-based SGD, to our knowledge, allowing a noisy oracle is not investigated before. In order to obtain our convergence bound, we assume $\cD_i(\xx)$  has zero-mean and bounded variance according to Assumption~\ref{ass:oracle-noise}.

\begin{assumption}[Zero-mean oracle noise with bounded variance]
	\label{ass:oracle-noise}
	There exists constants $P$ and $\zeta$ such that $\forall \xx \in \R^d$ and $\forall i \in [N]$:
	\begin{equation}
	 \E_{\xi_i\sim \cD_i(\xx)} \norm{\xi_i}_2^2 \leq \zeta^2 + P \norm{\nabla f(\xx)}_2^2 \, .
	\end{equation}
	Furthermore,
	\begin{equation}
	\E_{\xi\sim \cD_i(\xx)} [\,\xi\,] = 0 \, .
	\end{equation}
\end{assumption}

\begin{remark}%
	It is widely common in the literature \cite{mohtashami2021simultaneous,koloskova2020unified,Bottou2018:book} of SGD convergence theory to assume the  average oracle noise over the whole functions is bounded by $\hat{\zeta}^2 + \hat{P}\norm{\nabla f(\xx)}_2^2$ for some $\hat{\zeta}$ and $\hat{P}$. While here we make a slightly stronger assumption by assuming the bound applies to individual oracle noise, we note that this is still an improvement over previous work on permutation-based gradient descent which did not allow a noisy oracle assuming $\zeta = P = 0$. 
\end{remark}

\subsection{Epoch-Based Gradient Descent}
In order to allow a unified result for various variants of gradient descent, such as random reshuffling, and incremental gradient, we introduce a generalized template which we call epoch-based gradient descent. In this template the optimization is divided into epochs. In each epoch, a sequence $r_{t,1}, r_{t,2}, \ldots, r_{t,n}$ is chosen with elements from the set $[N] = \{1, 2, \ldots, N\}$ and the gradients of $f_{r_1}, f_{r_2}, \ldots, f_{r_n}$ are used, in this order, to update the parameters. The template is described in Algorithm~\ref{alg:epochgd} where we also provide examples of choosing this sequence to recover known algorithms such as random reshuffling but also sampling with replacement (SGD).

\begin{figure*}[t]
\resizebox{\linewidth}{!}{
\begin{tikzpicture}
	\node (alg) {%
		\begin{minipage}[l]{.6\linewidth }
\begin{algorithm}[H]
	\caption{Epoch-based Gradient Descent}
	\label{alg:epochgd}
	\begin{algorithmic}[1]
		\For{$t = 1 \ldots T$} %
		\State $\xx_{t}^{1} \gets \xx_{t - 1}$ 
		\State Determine the update sequence $r_t$.  
		\State Let $n$ be the length of $r_t$.
		\For{$i = 1 \ldots n$} %
		\State $\xx_{t}^{i + 1} \gets \xx_{t}^{i} - \gamma \nabla f_{r_{t,i}} (\xx_{t}^{i}) $
		\EndFor
		\State $\xx_{t} \gets \xx_{t}^{n + 1}$
		\EndFor
	\end{algorithmic}
\end{algorithm}
\end{minipage}};
\node (start) at (2.2cm , -.05cm + 1\baselineskip) {};
\draw[->, line width=.1cm,color=blue] (start)  -- node[midway, fill=white] {Examples} ($(start) + (3cm, 0)$)  node[pos=1.1,anchor=west] (examples) {};
\node[draw,anchor=north west, line width=.1mm] at ($(examples) + (-.25cm, 3\baselineskip)$) {\begin{minipage}[r]{.44\linewidth}%
		\paragraph{Random Reshuffling} Let $n = N$ and shuffle $[N]$ to obtain the sequence $r_{t, 1}, r_{t, 2}, \ldots, r_{t, N}$. \\
		\paragraph{Incremental GD} Let $n = N$ and define the sequence $r_{t,i} = i$ for $i \in [N]$.\\
		\paragraph{SGD} Let $n = 1$ and pick $r_{t,0}$ randomly from $[N]$.\\
	\end{minipage}};
\end{tikzpicture}}
\end{figure*}

 For the cases of random reshuffling or incremental gradient, $r$ would be a permutation of $[N]$ and therefore its length, $n$, would be equal to $N$. However, in general, $n$ does not have to be equal to the number of functions $N$. We allow this in order to also cover cases where a function is selected multiple times or not at all during an epoch. This can be useful for example when a subset of dataset is sampled at each epoch for optimization. Still, to ensure convergence we need certain guarantees that the selected subset is a good approximation of the global function $f$. This clearly holds for a permutation as the global function is exactly equal to the average of the selected subset. Here, we obtain this guarantee by making Assumption~\ref{ass:total-bias-sequence}.

	\begin{assumption}[Bounded sample bias]
		\label{ass:total-bias-sequence}
		We assume for all sequences $r$ chosen in an epoch, it holds that for all $\xx \in \R^d$:
		\begin{equation}
			\left\|\frac{1}{n}\sum_{i=1}^n \big(\nabla f_{r_i}(\xx) - \nabla f(\xx)\big)\right\|_2^2 \leq \frac{1}{4} \norm{\nabla f(\xx)}_2^2 \, .
		\end{equation}
	\end{assumption}

\begin{remark}
	In previous work, usually just the case of a permutation is considered. We emphasize that the assumption is satisfied when $r$ is a permutation of $[N]$ or a combination of several permutations. In this case, we have $\norm{\frac{1}{n}\sum_{i=1}^n (\nabla f_{r_i}(\xx) - \nabla f(\xx))}_2^2  = 0$. Hence, we make a more relaxed assumption here. %
	\end{remark}

For simplicity, we assume all sequences have the same length $n$. However, our proofs can be extended to cover cases where the lengths vary per epoch in which case $n$ should be equal to the maximum length of $r$.

\subsection{Quantifying the Effect of Order on Convergence}
We now move to quantifying the effect of the chosen sequence $r_t$ on convergence speed.  Let us introduce the following quantity for any sequence $r$ and any index $1 \leq k \leq n$ of that sequence at any $\xx \in R^d$:
\begin{equation*}
	\phi_{r, k}^2(\xx) := \left\|\sum_{i=1}^k (\nabla f_{r_{i}}(\xx) - \nabla f(\xx))\right\|_2^2 \,.
\end{equation*}
We now introduce the following assumption which assumes the values $\phi_{r_t, k}(\xx_t)$ are bounded.

\begin{assumption}[Bounded sequence heterogenity]
	\label{ass:heterogenity-sequence}
	There exists constants $M_\star$ and $\sigma_\star$ such that $\forall t \in [T]$ and $k \in [n]$:
	\begin{equation}
		\phi_{r_t, k}^2(\xx_t) \leq \sigma_\star^2 + k^2 M_\star \norm{\nabla f(\xx_t)}_2^2 \, ,
	\end{equation}
where the number of epochs $T$, number of steps in each epoch $n$, the parameters at $t$-th epoch $\xx_t$, and the order chosen in $t$-th epoch $r_t$ are defined as in Algorithm~\ref{alg:epochgd}.
\end{assumption}

\begin{remark}
	When there is randomness in choosing $r_t$, the above assumption can be changed to hold for the expectation over $r_1, r_2, \ldots, r_T$. In this case, the convergence bounds we will obtain will also hold in expectation.
\end{remark}

In the following, we propose using $\sigma_\star$ as a measure of the effect of order on convergence speed. 

We note that Assumption~\ref{ass:heterogenity-sequence} can be seen as a replacement for Assumption~\ref{ass:heterogenity} which is commonly used in the literature \cite{Bottou2018:book,pmlr-v130-stich21a,koloskova2020unified}. Note that we use the term heterogeneity to describe the variance between different functions while we use the term noise to describe the unstructured variance of oracle's output. Since both these terms refer to a variance, assumptions on their boundedness have a similar template. %

\begin{assumption}[Bounded dataset heterogenity]
	\label{ass:heterogenity}
	We call the set $\{g_1, g_2, \ldots, g_n\}$ of functions \textbf{$(\sigma, M)$-heterogenous} with respect to the function $g$ if $\forall \xx \in \R^d$: 
	\begin{equation}
		\frac{1}{n}\sum_{i=1}^n\norm{\nabla g_i(\xx) - \nabla g(\xx)}_2^2 \leq \sigma^2 + M \norm{\nabla g(\xx)}_2^2 \, .
	\end{equation}
	We assume there exists constants $M$ and $\sigma$ such that the set $\{f_1, \ldots, f_N\}$ is $(\sigma, M)$-heterogenous with respect to the function $f$.
\end{assumption}

Note that we do \textbf{not} need this assumption for our main result in Theorem~\ref{thm:epochgd-convergence}. However, we mention it here to allow comparison with prior work. We now make the following remarks about how it translates to Assumption~\ref{ass:heterogenity-sequence}: %
\begin{remark}
	\label{rmk:bounds_on_sigmastar}
	In general,  Assumption~\ref{ass:heterogenity} yields the bound $\frac{\phi_{r_t, k}^2(\xx_t)}{k^2} \leq \frac{n}{k} \sigma^2 + \frac{nM}{k}\norm{\nabla f(\xx)}_2^2$. Therefore under this assumption, Assumption~\ref{ass:heterogenity-sequence} holds with $\sigma_\star^2 = n^2 \sigma^2$ and $M_\star = nM$.
	For the case of random shuffling, \citet[Lemma 1]{mishchenko2020random} show that in $\E_{r_t} \frac{\phi_{r_t,k}^2(\xx_t)}{k^2} \leq \frac{n-k}{k(n - 1)} (\sigma^2 + M \norm{ \nabla f(\xx_t)}_2^2)\leq \frac{1}{k} \sigma^2 + \frac{M}{k} \norm{\nabla f(\xx)}_2^2$. This allows $\sigma_\star^2 = n \sigma^2$ and $M_\star = M$. 
\end{remark}

\subsection{Main Results}

We now state the following convergence bound for Algorithm~\ref{alg:epochgd}: 

\begin{theorem}
	\label{thm:epochgd-convergence}
	If Assumptions~\ref{ass:lsmooth}, \ref{ass:oracle-noise}, \ref{ass:total-bias-sequence}, and \ref{ass:heterogenity-sequence} are satisfied, and $\gamma < \frac{1}{8Ln(M_\star+\frac{P}{n}+1)}$, for iterates of Algorithm~\ref{alg:epochgd} it holds
	\begin{align*}
		\frac{1}{T}\sum_{t=1}^T \norm{\nabla f(\xx_t)}_2^2 \leq \frac{8F_0}{nT\gamma} + 16\gamma L \zeta^2 + 32\gamma^2L^2\sigma_\star^2
	\end{align*}
	which when carefully selecting $\gamma$ yields 
	\begin{align*}
		&\frac{1}{T}\sum_{t=1}^T \norm{\nabla f(\xx_t)}_2^2 \in \\& \ \ \ \ \cO\left( \frac{L(M_\star+\frac{P}{n}+1)F_0}{T} + \left(\frac{\sigma_\star LF_0}{nT}\right)^\frac{2}{3} + \zeta\sqrt{\frac{LF_0}{nT}}\right)
	\end{align*}
\end{theorem}

The proof of this theorem closely follows the proof used for convergence of random-reshuffling in \cite{mishchenko2020random} but covers more generalities such as oracle noise, more relaxed assumptions such as Assumption~\ref{ass:total-bias-sequence} or more general ones such as allowing $M_\star > 0$ in Assumption~\ref{ass:heterogenity-sequence}. We postpone the proof to Appendix~\ref{app:proof_epochgd_convergence} and continue to show some immediate results of this theorem including recovering previous results.

\subsection{Discussion}\label{sec:main_results_discussion} %

\paragraph{Recovering Bounds on Random Reshuffling and Incremental GD}

Previous work use Assumption~\ref{ass:heterogenity} to derive bounds for random reshuffling and incremental gradient descent. In Remark~\ref{rmk:bounds_on_sigmastar}, we established the values for $\sigma_\star$ and $M_\star$ when this assumption holds both in the general case and when choosing a random permutation. Using these values, we obtain the following bounds for random reshuffling (RR) and incremental gradient descent (IG).
\begin{align}
	\tag{RR}
	\cO\left( \frac{L(M + \frac{P}{n}+1)F_0}{T} + \left(\frac{\sigma LF_0}{\sqrt{n}T}\right)^\frac{2}{3} + \zeta\sqrt{\frac{LF_0}{nT}}\right)\, ,\\
	\tag{IG}
	\cO\left( \frac{L(M + \frac{P}{n}+1)F_0}{T} + \left(\frac{\sigma LF_0}{T}\right)^\frac{2}{3} + \zeta\sqrt{\frac{LF_0}{nT}}\right) \, .
\end{align}
Note that \citet{mishchenko2020random} derive their bounds assuming $M = 0$ in Assumption~\ref{ass:heterogenity} while setting $P$ and $\zeta$ in Assumption~\ref{ass:oracle-noise} to zero. In this case the convergence bound for non-convex smooth case in \cite{mishchenko2020random} is recovered for both algorithms.

\paragraph{Comparison with SGD}
We use SGD to refer to sampling with replacement. Consider minimizing a finite-sum problem with a single function $f_1$ having a noisy oracle that returns the noisy gradient for one of $h_1, \ldots, h_N$ randomly. This is equivalent to running SGD over $h_1, \ldots, h_N$. Similar to previous work, we assume Assumption~\ref{ass:heterogenity} holds for $h_i$. This means Assumption~\ref{ass:oracle-noise} holds for the oracle of $f_1$ with $\zeta = \sigma$ and $P = M$. Also, since there is only a single function $f_1$ 
we have $\sigma_\star = M_\star = 0$. Using these constants in Theorem~\ref{thm:epochgd-convergence}
 recovers previous results. Having a framework that can cover both random reshuffling and SGD captures the trade-off between using these methods. For example, it is also possible to consider dividing the functions to two groups, alternating between the two groups during training while each time picking a random function from the current group. We explore the usefulness of using similar schemes more in Section~\ref{sec:hierarchical_shuffling}. 

\paragraph{Sub-sampling}
While this is not our main focus in this work, we would like to note that given Theorem~\ref{thm:epochgd-convergence}, Assumption~\ref{ass:total-bias-sequence} provides a sufficient condition when a subset of dataset can be used with guaranteed convergence. This can be useful to analyze schemes that rely on using a subset of data. For example, using only easy samples at the beginning of curriculum learning while using the full data-set in the later stages might be justified since the easy samples might provide a good approximation when the gradient norm is large. We leave more investigation into this as a future work and return our focus to cases when a permutation of dataset is iterated at each epoch.

\section{Finding Good Permutations}
\label{sec:good_order}

\subsection{Effectiveness of Choosing the Right Order}
\label{sec:good_order_effectiveness}
In Theorem~\ref{thm:epochgd-convergence}, the order affects the convergence speed through parameter $\sigma_\star$. While our framework does not provide a lower bound, this still can hint that an order with a lower $\sigma_\star$ would converge faster. While we know $\sigma_\star \leq n^2\sigma^2$ when Assumption~\ref{ass:heterogenity} holds (see Remark~\ref{rmk:bounds_on_sigmastar}), it is not clear how good an order can be. This question allows assessing how important it is to choose the right order. To answer this question, in Appendix~\ref{app:lower_bound_sigmastar}, we show that $\sigma_\star^2 \in \Omega(\sigma^2)$. Therefore, when Assumption~\ref{ass:heterogenity} holds, the order can change the convergence speed up to a factor of $n^2$. In order to show that there exists cases when this lower bound is achieved, consider the example used in \cite{safran_how_2020} for establishing a lower bound for RR where
\begin{equation}
	\label{eq:example_function}
	f_i = \begin{cases}
		\frac{x^2}{2} + \sigma x & \text{if } i \leq \frac{N}{2}\\
		\frac{x^2}{2} - \sigma x & \text{if } i > \frac{N}{2}
	\end{cases} \, .
\end{equation}  
In this case, the order $f_1,  f_{\frac{n}{2}+1},  f_2, f_{\frac{n}{2}+2},\ldots, f_{\frac{n}{2}}, f_n$,  achieves $\sigma_\star^2 = \sigma^2$. On the other hand, using Lemma~12 of \cite{rajput_closing_2020} one can show that when choosing the order randomly, $\E \sigma_\star^2 \geq \frac{n \sigma^2}{256}$. Therefore, in this case, choosing the right order strictly improves the bound in Theorem~\ref{thm:epochgd-convergence} over random reshuffling.

While we established the possibility of beating random reshuffling by using an order with a smaller $\sigma_\star$, it is not clear how to find such order in the general case. We now provide a heuristic algorithm to find such orders by trying to minimize $\sigma_\star$.

\subsection{Algorithm for Finding Good Permutations}

Our algorithm aims to find an order with a low $\sigma_\star$ for the next epoch in a greedy manner. In particular, for the next function to use in the next step of epoch $t$, the algorithm chooses one of the remaining functions that minimizes $\phi_k^2(\xx_t)$ where $\xx_t$ is the parameters at the beginning of $t$-th epoch and $\phi$ is defined as in Assumption~\ref{ass:heterogenity-sequence}. The algorithm is described in Algorithm~\ref{alg:greedyalg}. 

\begin{algorithm}
	\caption{Greedy Order Chooser}
	\label{alg:greedyalg}
	\begin{algorithmic}
		\For{$i = 1 \ldots N$}
		\State Compute $\gg_t^i$ by querying the oracle for $\nabla f_i(\xx_t)$.
		\EndFor
		\State Compute the full gradient $\gg_t := \frac{1}{N} \sum_{i=1}^N \gg_t^i$.
		\State Let $\cc_{\rm bias} \gets \0$ \hfill $\triangleright$  the current bias of the found order
		\For{$i = 1 \ldots N$}
		\State Let $S$ be the subset of $[N]$ not in $r_1, \ldots r_{i - 1}$. %
		\State Find $r_i := \argmin_{i \in S} \norm{\gg_t^i - \gg_t + \cc_{\rm bias}}_2^2$.
		\State Let $\cc_{\rm bias} \gets \cc_{\rm bias} + \gg_t^{r_i} - \gg_t$.
		\EndFor
		\State Return $r$. %
	\end{algorithmic}
\end{algorithm}

We note that when this algorithm is used on the example functions (\ref{eq:example_function}) we considered in Section~\ref{sec:good_order_effectiveness}, it is able to find the right order, alternating between $f_i$ with $i \leq \frac{n}{2}$ and $i > \frac{n}{2}$. We showcase the effectiveness of our algorithm in more practical settings in Section~\ref{sec:exp_good_order} by applying it for training a neural network on CIFAR10 image classification task.

\section{Structured Shuffling}
\label{sec:hierarchical_shuffling}
\subsection{Motivation}
\label{sec:hierarchical_shuffling_motivation}

In our setup, we allowed accessing $\nabla f_i$ through a noisy oracle. As a result, it is possible to analyze cases where $f_i$ represents a group of data points but the oracle computes the gradient for one of them, possibly for efficiency purposes. An example is when $f_i$ corresponds to different augmentations of the same data point. 

As an another example, consider training a model in federated learning settings with $N$ workers where at each step only a single worker is active, chosen according to a random permutation. Assume that, similar to practice, when the worker is queried for a gradient, it returns an approximation by computing the gradient for a single data point. One option is that the worker chooses this data point independently at random in which case we can model this behavior as oracle noise similar to the case of data augmentation. We refer to this case as \textbf{running SGD internally}.

Alternatively, and more closer to practice, the worker can return the gradient for next data point in a random permutation over its dataset. In this case, the training goes through a permutation of the combined dataset of all workers. However, this permutation is not chosen completely at random and has a structure we call \textbf{two-level shuffling}. 

Alternatively, if there are $m$ data points in each worker, one can use a permutation containing $m$ copies of each $f_i$ for selecting the active worker in the next $mN$ steps. This is in contrast to two-level shuffling which goes through $m$ random permutations each containing a single copy of $f_i$. In this case, the permutation over the combined dataset is completely random, resembling random reshuffling. We refer to this case as \textbf{standard shuffling}.

Note that while these structures can be observed more naturally in the federated learning settings, they can also appear on single node trianing when grouping the data points according to an abstraction inherent to the problem. An example is grouping the data points according to their label. In this section, we investigate the effectiveness of each of these structures. In particular, we now additionally assume that each $f_i$ is the average of several other functions. Formally, we assume there exists function $h_{i, j}$ for each $i \in [N]$ and $j \in [m]$ such that
\begin{equation*}
	f_i(\xx) := \frac{1}{m}\sum_{j=1}^m h_{i,j}(\xx) \, .
\end{equation*}
For simplicity and brevity, we do not assume a noisy oracle for accessing the gradients of $h_{i,j}$. However, obtaining the results with such noise would be trivial, using Theorem~\ref{thm:epochgd-convergence} and  exactly the same methods we use in this section.

\subsection{Two-Level Shuffling}

We now define two-level shuffling in a formal way and generalize it so that each time $f_i$ is selected, it would perform $K$ steps using $f_i$, each time applying the gradient for next $h_{i,j}$. This would be similar to performing $K$ local steps each time a worker is selected in federated learning settings we considered before. We refer to this as two-level $K$-shuffling. The example we considered in Section~\ref{sec:hierarchical_shuffling_motivation} corresponds to the case when $K = 1$.  Algorithm~\ref{alg:kshuffling} generates a two-level $K$-shuffling. %

\def\topr{r^{top}}
\def\lowr{r^{low}}
\def\combinedr{r^{full}}

\begin{algorithm}
	\caption{Two-Level $K$-Shuffling}
	\label{alg:kshuffling}
	\begin{algorithmic}
		\State Let $\combinedr$ be the final ordering and initialize it to an empty list.
		\State For each $i$, determine $\lowr_{i}$, the internal update sequence of $f_i$ (e.g. a permutation of $[m]$).
		\For {$k = 1 \ldots \frac{m}{K}$}
			\State Determine an order of top-level functions, $\topr_{k}$ (e.g. a permutation of $[N]$).
			\For {$i = 1 \ldots N$}
				\State Append the next $K$ elements in $\lowr_{\topr_{k, i}}$ (i.e. the pairs $(\topr_{k, i}, \lowr_{, \topr_{k, i}, m(k - 1) + 1}) \ldots (\topr_{k, i}, \lowr_{\topr_{k, i}, mk})$) to $\combinedr$.
			\EndFor
		\EndFor
		\State Return $\combinedr$.
	\end{algorithmic}
\end{algorithm}

\begin{remark}
In  Algorithm~\ref{alg:kshuffling} and the following, we assume $K$  is a divisor of $m$. However, this assumption can be easily avoided and is only for simplicity. One way of avoiding this is increasing $m$ to become a divisor of $K$ by adding additional no-op second-level functions to each $f_i$ so that when these functions are selected the algorithm does nothing (e.g. the returned gradient is zero).  Also note that while the algorithm is stated here for two levels, extending it to multiple levels is trivial and can be done by applying Algorithm~\ref{alg:kshuffling} also for obtaining the low-level orders $\lowr$.
\end{remark}

\subsection{Comparing Two-Level Shuffling and Standard Shuffling}

\def\topsigma{\sigma_{top}}
\def\lowsigma{\sigma_{low}}
\def\combinedsigma{\sigma_{full}}

We assume $\{f_1, \ldots, f_N\}$ is $(\topsigma, M)$-heterogeneous (as defined in Assumption~\ref{ass:heterogenity}) with respect to $f$ while for each $i \in [N]$, the set of functions $\{h_{i, 1}, \ldots, h_{i, m}\}$ is $(\lowsigma, M)$-heterogeneous with respect to $f_i$. Since these two assumptions hold, it can be seen that the set of all functions $\{h_{1,1}, \ldots, h_{N, m}\}$ is $(\combinedsigma, M)$-heterogeneous with respect to $f$ for some $\combinedsigma$ where  $\combinedsigma^2\leq \topsigma^2 + \lowsigma^2$.

We will now analyze the effect of using two-level $K$-shuffling on convergence when $\topr$ and $\lowr$ in Algorithm~\ref{alg:kshuffling} are random permutations of $[N]$ and $[m]$. We can directly use Theorem~\ref{thm:epochgd-convergence} by noticing the following lemma:

\begin{lemma}
		\label{lemma:variance_hierarchical_shuffling}
	If  $\combinedr$ is a sequence obtained from Algorithm~\ref{alg:kshuffling} when $\topr$ and $\lowr$ are random permutations,
	\begin{align*}
		\E \sigma_\star^2(\combinedr) \in \cO( NK^2\topsigma^2 + (m+N(m-K)) \lowsigma^2) \, .
		\end{align*}
	\end{lemma}

We postpone the proof to Appendix~\ref{app:proof_variance_hierarchical_shuffling} and proceed to discussing the following corollary which follows directly from combining Lemma~\ref{lemma:variance_hierarchical_shuffling} and Theorem~\ref{thm:epochgd-convergence}.

\begin{corollary}
	When $\{f_i\}$ is $(\topsigma, M)$-heterogeneous with respect to $f$ while $\{h_{i, j}\}$ is $(\lowsigma, M)$-heterogeneous with respect to $f_i$, running Algorithm~\ref{alg:epochgd} with orders generated from Algorithm~\ref{alg:greedyalg} using random permutations of $[N]$ and $[m]$ for $\topr$ and $\lowr$ yields the convergence bound
\begin{align}
	\label{eq:kshuffle-bound}
&\cO\Bigg( \frac{L(M_\star+1)F_0}{T} +  \\&\ \ \ \ \  \left(\frac{(K\topsigma +\sqrt{m-K}\lowsigma )LF_0}{\sqrt{N}mT}\right)^\frac{2}{3} + \left(\frac{\lowsigma LF_0}{N\sqrt{m}T}\right)^\frac{2}{3} \Bigg) \nonumber
\end{align}
\end{corollary}

In contrast, standard shuffling yields
\begin{equation}
	\label{eq:normalshuffle-bound}
\cO\left( \frac{L(M_\star+1)F_0}{T} + \left(\frac{\combinedsigma LF_0}{\sqrt{Nm}T}\right)^\frac{2}{3}\right)
\end{equation}
When $\combinedsigma \in \Theta(\topsigma + \lowsigma)$,  setting $K$ to $\sqrt{m}$ yields the same bound when using $K$-shuffling as when using simple shuffling. However, for example in federated learning, using $K = \sqrt{m}$ reduces the communication costs by $\sqrt{m}$, making $\sqrt{m}$-shuffling a better alternative to simple shuffling. Moreover, $K$-shuffling can yield improved rates in other cases by tuning $K$. This can be especially observed for the case when either $\topsigma \ggg \lowsigma$ by setting $K = 1$ or $\topsigma \lll \lowsigma$ by setting $K = m$ obtaining a better bound by a factor of $\sqrt{m}$ or $\sqrt{N}$ respectively. In federated learning, these correspond to doing a single local step when workers have a highly heterogeneous distribution while doing $m$ local steps when they are homogeneous.

We verify the effectiveness of using $1$-shuffling in practice in Section~\ref{sec:exp-hierarchical-shuffling}. Here, we continue by comparing running SGD internally, i.e. sampling with replacement from $h_{i, j}$ as defined in Section~\ref{sec:hierarchical_shuffling_motivation},  against shuffling.

In this case, the variance in the lower level becomes part of the noise of the oracle for $f_i$s and therefore decays slower than when using shuffling. On the other hand, the upper bound on the learning rate in Theorem~\ref{thm:epochgd-convergence} becomes larger, allowing faster convergence. In particular, the convergence bound obtained from Theorem~\ref{thm:epochgd-convergence} after $NmT$ steps is
\begin{equation}
	\label{eq:shuffle-just-high-level}
	\cO\left( \frac{L(M_\star+1)F_0}{mT} + \left(\frac{\topsigma LF_0}{\sqrt{N}mT}\right)^\frac{2}{3} + \lowsigma \sqrt{\frac{LF_0}{NmT}}\right) \, .
\end{equation}
In comparison with (\ref{eq:normalshuffle-bound}), it can be seen that the first term and the term containing $\topsigma$ are becoming smaller faster, while the term containing $\lowsigma$ decays slower. This can be preferred when the term containing $\lowsigma$ is much smaller than the other terms. This can especially happen at the early stages of training.  In particular, when $T < \frac{LF_0 \sqrt{nM}}{\lowsigma^2}$, the last term in (\ref{eq:shuffle-just-high-level}) is less than both the first and last term in~(\ref{eq:kshuffle-bound}). 

  This discussion can point to the fact that the best convergence speed might be obtained by using a combination of these cases, possibly starting with running SGD internally with a large learning rate and gradually switching to shuffling while decaying the learning rate.  %
   This resembles the learning rate decay method widely utilized in practice \cite{you2019does}.

\section{Experiments}
\subsection{Finding Good Orders with Algorithm~\ref{alg:greedyalg}}
\label{sec:exp_good_order}

According to the theoretical arguments in Section~\ref{sec:good_order}, it can be observed that the gap between using Algorithm~\ref{alg:greedyalg} and random reshuffling depends on the variance of functions, $\sigma^2$. However, this value is not necessarily large enough for us to observe a large gap in practice especially due to mini-batching which further reduces the variance. In order to avoid this problem, we use a special mini-batching scheme, called \textbf{same-class batching}. 

In same-class batching, data points belonging to each class are separately batched so each batch contains data points from the same class. Moreover, these batches are kept fixed throughout training. This is unlike the \textbf{standard batching} where consecutive elements of a random permutation of the dataset are batched together. The permutation is chosen again at random at the beginning of each epoch. 

Note that if $f_i$ corresponds to the average of loss on data points belonging to the $i$-th class, using same-class batching with batch size $\tau$ would keep the variance between batches equal to $\sigma^2$ while using standard batching would reduce it by a factor of $\tau$. 

Since batch normalization layers \cite{ioffe2015batch} depend on the batch average during training, we suspect that they might introduce side effects when combined with same-class batching and therefore do not use them in our experiments. 

We now train a ResNet-18 without Batch Normalization on CIFAR10 \cite{krizhevsky2009learning}  dataset both with and without same-class batching utilizing both random reshuffling, the standard method widely used for training neural networks, or Algorithm~\ref{alg:greedyalg} to find a good order for the next epoch. We train the models for 200 epochs using SGD optimizer with $0.9$ momentum, decaying the learning rate after epochs $80$, $120$, and $160$ by a factor $0.1$. The initial learning rate is set to $0.1$ for standard batching and to $0.01$ for same-class batching. The experiments are repeated 3 times with different random seeds. The test accuracies at the end of training are reported in Table~\ref{tab:cifar10_single_class_batch}. The trajectory of test accuracy during training is also plotted in Appendix~\ref{app:fig_cifar10_single_class_batch}.  While the Greedy algorithm is able to obtain similar accuracy to RR with standard batching, it significantly outperforms RR when same-class batching is used. 

Algorithm~\ref{alg:greedyalg} requires computing a full gradient which is computationally expensive. Therefore, we additionally consider changing the order every 10 epochs instead of at every epoch. This leads to a small drop in accuracy but reduces the computation cost noticeably.

\begin{table}
	\caption{Test accuracies at the end of training when using RR and Algorithm~\ref{alg:greedyalg} (denoted by Greedy) with standard or same-class batching. Slow Update refers to the case when the order is only updated (using Algorithm~\ref{alg:greedyalg}) every 10 epochs. The average of 3 runs is reported with the standard error stated inside parantheses.}
	\label{tab:cifar10_single_class_batch}
	\centering
	\begin{tabular}{ccc}
		\toprule
		&Same-Class&Standard\\\midrule
		RR&\textbf{92.72(0.16)}&73.46(0.34)\\ %
		Greedy&\textbf{92.67(0.54)}&\textbf{88.43(0.07)}\\
		Slow Update&89.51(0.72)&83.62(0.52)\\\bottomrule
	\end{tabular}
\end{table}

\subsection{Two-Level Shuffling}
\label{sec:exp-hierarchical-shuffling}
We established that using two level shuffling, better convergence bounds can be obtained under mild conditions. In order to show that this is also true in practice, we compare the performance of standard shuffling and two-level shuffling on a set of functions for which we can control $\topsigma$ and $\lowsigma$. We define the global function $f: \R^d \to R$ to be
\[ 
f(\xx) := \frac{1}{2} \lin{A\xx, \xx} + \frac{\lambda}{2} \norm{\xx}^2  \, ,
\]
with  $d = 20$, $\lambda = 0.2$, and band-diagonal matrix $A$ with  $[-\1_{d-1}, 2\1_{d}, -\1_{d-1}] $ on the diagonals. This is a challenging problem without regularization \citep{Nesterov2004:book}. We now define two top-level functions $f_1$ and $f_2$, and subsequently $h_{i, j}$ to be:
\begin{align*}
f_1(\xx) &:= f(\xx) + \topsigma  \lin{\1, \xx}\\
f_2(\xx) &:= f(\xx) - \topsigma \lin{\1, \xx}
\\
h_i, j(\xx) &:=	\begin{cases}
f_i(\xx) + \lowsigma \lin{\1, \xx} & \text{if } j \leq \frac{m}{2}\\
f_i(\xx) - \lowsigma \lin{\1, \xx} &\text{if } j > \frac{m}{2}
		\end{cases}
\end{align*}
We set $\lowsigma = 10$ and compare the number of steps required to obtain accuracy $\norm{\xx}_2 < 0.2$ for different values of $\topsigma$ and $m$, the number of functions corresponding to each $f_i$.  For each pair and each order generating method, we tune the learning rate over the grid $1.1 \cdot \{2^{-1}, \ldots, 2^{-20}\}$ based on performance over 3 runs.  The result is plotted in Figure~\ref{fig:two-level-shuffling}. It can be observed that as the value of $\topsigma$ grows, two-level shuffling clearly outperforms simple shuffling, confirming our theoretical arguments in practice. %

\begin{figure}
	\centering
	\includegraphics[width=.9\linewidth]{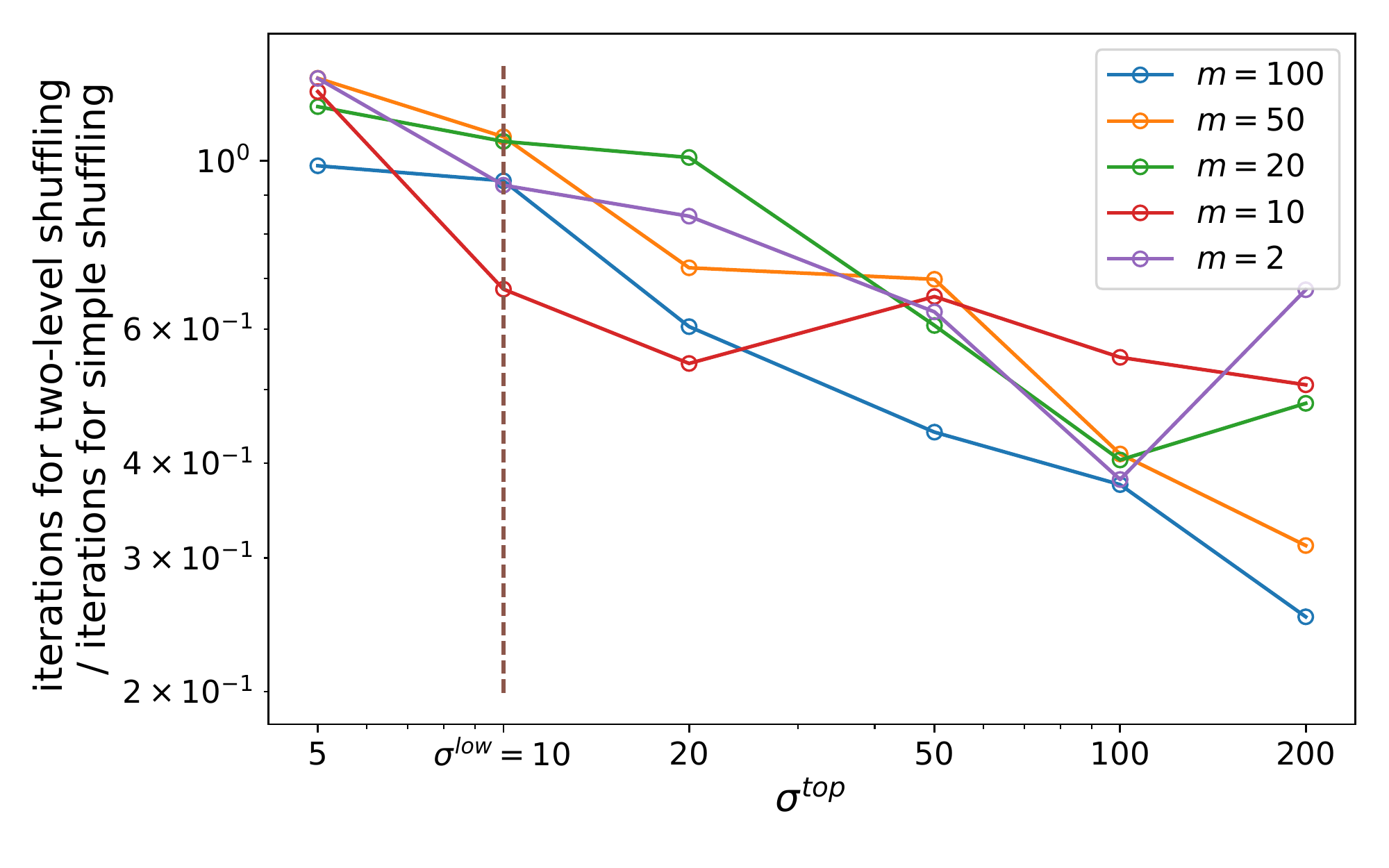}\vspace{-.5cm}
	\caption{Ratio of number of steps required to reach $\norm{\xx}_2 < 0.2$ when using two-level shuffling to when using standard shuffling for different values of $\topsigma$ and $m$. Superiority of two-level shuffling is clear especially for larger values of $\topsigma$.}
	\label{fig:two-level-shuffling}
\end{figure}
  
  \section{Future Work}
  While we have presented the efficacy of Algorithm~\ref{alg:greedyalg}, running this algorithm at epoch is costly. We also showed an alternative by running this algorithm every few epochs but observed a small impact on accuracy. Finding efficient algorithms for finding good orders obtaining the same accuracy is therefore grounds for future work. Exploring the effectiveness of two-level shuffling in other practical scenarios is also of interest. From a theoretical perspective, currently there is a gap between the best lower bound $\cO(\frac{1}{N^3T^2})$ and the best upper bound $\cO(\frac{1}{N^2T^2})$ achievable by optimizing order in our framework. Filling this gap can also be grounds for future work.

\section{Conclusion}
We introduced a measure for quantifying the effect of order on convergence of gradient descent, establishing a framework that yields a convergence bound tailored to any custom ordering. By designing an algorithm for selecting good orders and showing its effectiveness in practice, we exampled the practical usefulness of our framework. Finally, we demonstrated that using structured shuffling can be beneficial, proposing two-level shuffling and showing its superiority in theoretical and practical settings.%

{\small
\bibliographystyle{plainnat}
\bibliography{references}
}

\clearpage
\appendix
\onecolumn

\section{Proof of Theorem~\ref{thm:epochgd-convergence}}
\label{app:proof_epochgd_convergence}
\begin{lemma}
	\label{lemma:mischenko-after-epoch-decrease}
	If Assumptions~\ref{ass:lsmooth}, \ref{ass:oracle-noise}, and \ref{ass:total-bias-sequence} hold, and $\gamma < \frac{1}{2L(n+P)}$, we have 
	\begin{equation}
		\E_{\xi_t} f(\xx_{t + 1})  \leq f(\xx_t) - \frac{\gamma n}{8}\norm{\nabla f(\xx_t)}_2^2 - \frac{\gamma n}{4} \norm{\frac{\gg_t}{n}}_2^2 + \gamma L^2 V_t + \frac{\gamma^2Ln}{2}\zeta^2 \,,
	\end{equation}
	where $V_t =\sum_{i=1}^n\norm{\xx_t - \xx_t^i}_2^2$.
\end{lemma}
\begin{proof}
	Let $g_t := \sum_{i=1}^n \nabla f_{r_i}(\xx_t^i)$ and $\xi_t$ represent the sequence of noises for the $t$-th epoch.
	\begin{align*}
		\E_{\xi_t} f(\xx_{t + 1}) & \leq \E_{\xi_t} f(\xx_t) + \E_{\xi_t} \lin{\nabla f(\xx_t), x_{t + 1} - x_t} + \E_{\xi_t}\frac{L}{2}\norm{\xx_{t+1} - \xx_t}\\
		& = f(\xx_t) - \gamma n \lin{\nabla f(\xx_t), \frac{\gg_t}{n}} + \frac{\gamma^2Ln^2}{2}\norm{\frac{\gg_t}{n}}^2 + \frac{\gamma^2L}{2}\norm{\sum_{i=1}^n\xi_t^i}_2^2\\
		& \stackrel{(1)}{=} f(\xx_t) - \gamma n \lin{\nabla f(\xx_t), \frac{\gg_t}{n}} + \frac{\gamma^2Ln^2}{2}\norm{\frac{\gg_t}{n}}^2 + \frac{\gamma^2L}{2}\sum_{i=1}^n\norm{\xi_t^i}_2^2\\
		& = f(\xx_t) - \frac{\gamma n}{2}\big(\norm{\nabla f(\xx_t)}_2^2 + \norm{\frac{\gg_t}{n}}_2^2 - \norm{\nabla f(\xx_t) - \frac{\gg_t}{n}}_2^2\big) + \frac{\gamma^2Ln^2}{2}\norm{\frac{\gg_t}{n}}_2^2 \\ & \, \, \, \, \, \,  + \frac{\gamma^2L}{2}\sum_{i=1}^n\norm{\xi_t^i}_2^2\\
		& = f(\xx_t) - \frac{\gamma n}{2}\norm{\nabla f(\xx_t)}_2^2 - \frac{\gamma n}{2} (1 - L\gamma n) \norm{\frac{\gg_t}{n}}_2^2 + \frac{\gamma n}{2} \norm{\nabla f(\xx_t) - \frac{\gg_t}{n}}_2^2\\ & \, \, \, \, \, \,  + \frac{\gamma^2L}{2}\sum_{i=1}^n\norm{\xi_t^i}_2^2\\
		& \leq  f(\xx_t) - \frac{\gamma n}{2} (1 - \gamma L P)\norm{\nabla f(\xx_t)}_2^2 - \frac{\gamma n}{2} (1 - L\gamma n) \norm{\frac{\gg_t}{n}}_2^2 + \frac{\gamma n}{2} \norm{\nabla f(\xx_t) - \frac{\gg_t}{n}}_2^2\\ & \, \, \, \, \, \,  + \frac{\gamma^2Ln}{2}\zeta^2\\
	\end{align*}
	where in (1) we used the independence of $\zeta_t^i$. We now proceed by bounding the distance of average of updates to the true gradient:
	\begin{align*}
		\norm{\nabla f(\xx_t) - \frac{\gg_t}{n}}_2^2 &= \norm{\frac{1}{n}\sum_{i=1}^n [\nabla f_{r_i}(\xx_t) - \nabla f_{r_i}(\xx_t^i)] + \frac{1}{n}\sum_{i=1}^n [\nabla f_{r_i}(\xx_t) - \nabla f(\xx_t)]}_2^2\\
		& \leq 2 \norm{\frac{1}{n}\sum_{i=1}^n [\nabla f_{r_i}(\xx_t) - \nabla f_{r_i}(\xx_t^i)]}_2^2 + 2\norm{\frac{1}{n}\sum_{i=1}^n [\nabla f_{r_i}(\xx_t) - \nabla f(\xx_t)]}_2^2\\
		& \leq \frac{2}{n} \sum_{i=1}^n \norm{\nabla f_{r_i}(\xx_t) - \nabla f_{r_i}(\xx_t^i)}_2^2 + 2\phi_n^2(\xx_t)\\
		& \leq \frac{2}{n}\sum_{i=1}^nL^2\norm{\xx_t - \xx_t^i}_2^2 + 2\phi_n^2 \\
		&  = \frac{2L^2}{n}V_t + 2\phi_n^2 (\xx_t)\\
		& \leq \frac{2L^2}{n}V_t + \frac{1}{4} \norm{\nabla f(\xx_t)}_2^2
	\end{align*}
	Where the last inequality holds because of Assumption~\ref{ass:total-bias-sequence}. Using $\gamma < \frac{1}{2L(n+P)}$ and applying the bound we derived above, we get
	\begin{align*}
		\E_{\xi_t} f(\xx_{t + 1}) & \leq f(\xx_t) - \frac{\gamma n}{8}\norm{\nabla f(\xx_t)}_2^2 - \frac{\gamma n}{4} \norm{\frac{\gg_t}{n}}_2^2 + \gamma L^2 V_t + \frac{\gamma^2Ln}{2}\zeta^2\\
	\end{align*}
\end{proof}

\begin{lemma}
	\label{lemma:bounding-vt}
	If Assumptions~\ref{ass:lsmooth}, \ref{ass:oracle-noise}, and \ref{ass:heterogenity-sequence} are satisfied, and $\gamma \leq \frac{1}{2L(n+P)}$
	\begin{equation}
		\sum_{k=1}^n \norm{\xx_t^k - \xx_t}_2^2 \leq 4 \gamma^2 n^3 \norm{f(\xx_t)}_2^2 + 4 \gamma^2n\sigma_\star^2
	\end{equation}
\end{lemma}
\begin{proof}
	\begin{align*}
		\norm{\xx_t^k - \xx_t}_2^2 & = \gamma^2 \norm{\sum_{i=1}^{k - 1} \nabla f_{r_i}(\xx_t^i)}_2^2 + \gamma^2 \norm{\sum_{i=1}^{k-1}\xi_t^i}_2^2 \\
		&\leq \gamma^2 \norm{\sum_{i=1}^{k - 1} \nabla f_{r_i}(\xx_t^i)}_2^2 + \gamma^2  P  \sum_{i=1}^{k - 1}\norm{\nabla f(\xx_t^i)}_2^2 + \gamma^2 (k - 1) \zeta^2\\
		&\leq 2\gamma^2 \norm{\sum_{i=1}^{k - 1} \nabla f_{r_i}(\xx_t^i) - \nabla f_{r_i}(\xx_t)}_2^2 + 2\gamma^2 \norm{\sum_{i=1}^{k - 1} \nabla f_{r_i}(\xx_t)}_2^2  + \\&\ \ \ \ \ 2\gamma^2  P\sum_{i=1}^{k-1} \norm{\nabla f(\xx_t^i) - \nabla f(\xx_t)}_2^2 + 2\gamma^2  P  \sum_{i=1}^{k-1}\norm{\nabla f(\xx_t)}_2^2  + \gamma^2 (k - 1) \zeta^2\\\\
		&\leq 2\gamma^2  (k - 1 + P)L^2 \sum_{i = 1}^{k - 1} \norm{\xx_t^i - \xx_t}_2^2 + 2\gamma^2\norm{\sum_{i=1}^{k - 1} \nabla f_{r_i}(\xx_t)}_2^2  + 2\gamma^2 P (k - 1) \norm{\nabla f(\xx_t)}_2^2 + \gamma^2 (k - 1) \zeta^2\\\\
		&\leq 2\gamma^2 (k - 1 + P)L^2 \sum_{i = 1}^n \norm{\xx_t^i - \xx_t}_2^2 + 2\gamma^2\norm{\sum_{i=1}^{k - 1} \nabla f_{r_i}(\xx_t)}_2^2  + 2\gamma^2 P (k - 1) \norm{\nabla f(\xx_t)}_2^2 + \gamma^2 n \zeta^2\\
	\end{align*}
	In order to bound the second term we can write
	\begin{align*}
		\norm{\sum_{i=1}^k \nabla f_{r_i}(\xx_t)}_2^2 &= \norm{(k - 1) \nabla f(\xx_t) + \sum_{i = 1}^{k - 1} (\nabla f_{r_i}(\xx_t) - \nabla f(\xx_t))}_2^2 \\
		& \leq 2(k - 1)^2 \norm{\nabla f(\xx_t)}_2^2 + 2\norm{\sum_{i=1}^{k - 1} (\nabla f_{r_i}(\xx_t) - \nabla f(\xx_t))}_2^2\\
		& \leq 2(k - 1)^2 (1 + M_\star)\norm{\nabla f(\xx_t)}_2^2 + 2\sigma_\star^2
	\end{align*}
	
	Putting the result back we get:
	\begin{align*}
		\norm{\xx_t^k - \xx_t}_2^2 &\leq 2\gamma^2 (k - 1 + P)L^2 \sum_{i = 1}^n \norm{\xx_t^i - \xx_t} + 4\gamma^2(k - 1)(k - 1 + \frac{P}{1 + M_\star})(1 + M_\star)\norm{\nabla f(\xx_t)}_2^2 \\ & \ \ \ \ \  + 4\gamma^2\sigma_\star^2 + \gamma^2 n \zeta^2
	\end{align*}
	
	Denoting $V_t := \sum_{k=1}^n \norm{\xx_t^k - \xx_t}_2^2$ and summing up for all $k$ we get:
	\begin{align*}
		V_t &\leq \gamma^2 n(n - 1 + 2P)L^2 V_t + \frac{2}{3}\gamma^2 (1+M_\star)(2n + 3\frac{P}{1+M_\star} -1)n(n - 1) \norm{\nabla f(\xx_t)}_2^2 + 2\gamma^2 n\sigma_\star^2  + \gamma^2 n^2 \zeta^2
	\end{align*}
	Therefore, applying $\gamma < \frac{1}{2L(n+P)}$, we get $\gamma^2 n(n - 1 + 2P)L^2 < \frac{1}{2}$ which yields
	\begin{align*}
		V_t &\leq \frac{4}{3}\gamma^2 (1+M_\star)(2n + 3\frac{P}{1+M_\star} - 1)n(n - 1) \norm{\nabla f(\xx_t)}_2^2 + 4\gamma^2 n\sigma_\star^2\ + 2\gamma^2 n^2 \zeta^2\\
		&\leq \frac{4}{3}\gamma^2 n^2(n + nM_\star + P)\norm{\nabla f(\xx_t)}_2^2 + 4\gamma^2 n\sigma_\star^2\ + 2\gamma^2 n^2 \zeta^2\,. \qedhere
	\end{align*}
\end{proof}

\begin{proof}[Proof of Theorem~\ref{thm:epochgd-convergence}]
	Using the bound of Lemma~\ref{lemma:bounding-vt} to the result of Lemma~\ref{lemma:mischenko-after-epoch-decrease} and applying $\gamma < \frac{1}{8L(n+nM_\star+P)}$ we get
	\begin{align*}
		\E_{\xi_t} f(\xx_{t+1}) & \leq f(\xx_t) - \frac{\gamma n}{8} (1 - \frac{32}{3}\gamma^2L^2n(n+nM_\star+P)) \norm{\nabla f(\xx_t)}_2^2
		- \frac{\gamma n}{4} \norm{\frac{\gg_t}{n}}_2^2 \\  & \ \ \ \ \ 
		+ \frac{\gamma^2Ln}{2}(1+4\gamma Ln)\zeta^2
		+ 4\gamma^3 L^2 n \sigma_\star^2\\
		& \leq \E f(\xx_t) - \frac{\gamma n}{12} \norm{\nabla f(\xx_t)}_2^2
		+ \gamma^2 Ln\zeta^2 
		+ 4\gamma^3 L^2 n \sigma_\star^2
	\end{align*}
	Rearranging we get:
	\begin{align*}
		\frac{1}{12} \cdot \norm{\nabla f(\xx_t)}_2^2 \leq \frac{f(\xx_t) - \E_{\xi_t} f(\xx_{t+1}) }{\gamma n} + \gamma L \zeta^2 + 4 \gamma^2L^2\sigma_\star^2
	\end{align*}
	We now denote $F_0 := F(\xx_0) - f*$ and take the average over all the steps while taking the expectation over $\xi_t$ for all $t$. This yields:
	\begin{align*}
		\frac{1}{12T} \sum_{i=1}^T \norm{\nabla f(\xx_t)}_2^2 \leq \frac{F_0}{\gamma n T} + \gamma L \zeta^2 + 4\gamma^2L^2 \sigma_\star^2
	\end{align*}
	Finding a good learning rate using Lemma 17 of \cite{koloskova2020unified} we get:
	\begin{align*}
		&\cO \left(\zeta\sqrt{ \frac{LF_0}{nT}} + \left(\frac{\sigma_\star LF_0}{nT}\right)^\frac{2}{3} + \frac{L(M_\star+\frac{P}{n}+ 1)F_0}{T}\right)\,. \qedhere
	\end{align*}
\end{proof}

\section{Proof of Lemma~\ref{lemma:variance_hierarchical_shuffling}}
\label{app:proof_variance_hierarchical_shuffling}
\begin{proof}
	
	\def\neff{n_{\text{eff}}}
	For brevity, let us misuse the notation and right $h_{\topr_i, \lowr_j}$ instead of $h_{\topr_i, \lowr_{\topr_i, j}}$. Also, let us define $r0$ and $r1$ such that $\combinedr_i =: (r0_i, r1_i)$. We can now write:
	
	\begin{align*}
		\E_{\combinedr} \phi_{\combinedr, (aN + b)K + c}^2(\xx) & = \E_{\combinedr}\norm{\sum_{i=1}^{(aN + b)K + c}\nabla h_{\combinedr_i}(\xx) - \nabla f(\xx)}_2^2 \\
		& \stackrel{(A)}{=} \E_{\combinedr}\norm{\sum_{i=1}^{(aN + b)K + c}\nabla h_{\combinedr_i}(\xx) - \nabla f_{r0_i}(\xx)}_2^2 + \E_{\topr}\norm{\sum_{i=1}^{(aN + b)K + c}\nabla f_{r0_i}(\xx) - \nabla  f(\xx)}_2^2 \\		
		& = \E_{\combinedr}\norm{\sum_{i=1}^{(aN+ b)K + c}\nabla h_{\combinedr_i}(\xx) - \nabla f_{r0_i}(\xx)}_2^2 \\&\ \ \ \ \ + \E_{\topr_a}\norm{c(\sum_{j=1}^b\nabla f_{\topr_{a,j}}(\xx) - \nabla  f(\xx)) + (\sum_{j=1}^{b-1}\nabla f_{\topr_{a,j}}(\xx) - \nabla  f(\xx)) }_2^2 \\
			& \leq \E_{\combinedr}\norm{\sum_{i=1}^{(aN + b)K + c}\nabla h_{\combinedr_i}(\xx) - \nabla f_{r0_i}(\xx)}_2^2 + 2\E_{\topr_a}c^2\phi_{\topr_a, b}+\phi_{\topr_a,b-1}\\
			& \leq  \E_{\combinedr}\norm{\sum_{i=1}^{(aN + b)K + c}\nabla h_{\combinedr_i}(\xx) - \nabla f_{r0_i}(\xx)}_2^2 + 2((K-1)^2+1)\E_{\topr_a}\sigma_\star^2(\topr_a)\\
			& \leq  \E_{\combinedr}\norm{\sum_{i=1}^{(aN + b)K + c}\nabla h_{\combinedr_i}(\xx) - \nabla f_{r0_i}(\xx)}_2^2 + 2NK^2\topsigma^2\\
			& \leq  \E_{\combinedr}\lVert\sum_{i=1}^{b} \sum_{j=1}^{(a+1)K} \nabla h_{\topr_{a,i}, \lowr_{j}}(\xx) - \nabla f_{\topr_{a,i}} 
				\\& \qquad \qquad + \sum_{j=1}^{aK+c} \nabla h_{\topr_{a,b+1}, \lowr_{j}}(\xx) - \nabla f_{\topr_{a,b+1}} 
				\\& \qquad \qquad +\sum_{i=b+2}^{N} \sum_{j=1}^{aK} \nabla h_{\topr_{a,i}, \lowr_{ j}}(\xx) - \nabla f_{\topr_{a,i}} \rVert_2^2 \\& \ \ \ \ \ + 2NK^2\topsigma^2\\
			& \stackrel{(B)}{=} \sum_{i=1}^b \E_{\combinedr} \phi_{\lowr_{\topr_{a, i}}, (a+1)K} + \E_{\combinedr} \phi_{\lowr_{\topr_{a,b+1}},aK+c} + \sum_{i=b+2}^N\E_{\combinedr} \phi_{\lowr_{\topr{a,i}},aK} + 2NK^2\topsigma^2\\
			&\leq \lowsigma^2 (b \frac{(a+1)K(m-(a+1)K)}{m-1} +   \\&\qquad\qquad \frac{(aK+c)(m-(aK+c))}{m-1} +  \\&\qquad\qquad  (N-b-1)\frac{aK(m-aK)}{m-1}) \\&\ \ \ \ \ + 2NK^2\topsigma^2\\
			&\stackrel{(C)}{\leq} \lowsigma^2 (N \frac{m(m-K)}{m-1} +  \frac{(c)(m-c)}{m-1})   \\&\ \ \ \ \  + 
			 \lowsigma^2 (N \frac{m(m-K)}{m-1} +  \frac{m(m-K)}{m-1} + N\frac{m(m-K)}{m-1})
			\\&\ \ \ \ \ + 2NK^2\topsigma^2\\
			& \leq 3m (1 + N(1 - \frac{K-1}{m-1})) \lowsigma^2 + 2NK^2\topsigma^2
	\end{align*}

	where in (A) and (B) we used the randomness of $\lowr$ and $\topr$. (C) can be obtained by separating the cases when $a = 0$ and $a \geq 1$ and noticing $ a \leq \frac{m}{K} - 1$, $b \leq N - 1$ and $c \leq K - 1$.
\end{proof}

\section{Lower bound on $\sigma_*$}
\label{app:lower_bound_sigmastar}
Using Assumption~\ref{ass:heterogenity-sequence}, we can write:
\begin{align*}
	\norm{\nabla f_{r_k}(\xx) - \nabla f(\xx)}_2^2 & = \left\|\sum_{i=1}^k (\nabla f_{r_i}(\xx) - \nabla f(\xx)) - \sum_{i=1}^{k - 1}(\nabla f_{r_i} - \nabla f(\xx))\right\|_2^2\\
	& \leq 2\phi_{r,k}^2 + 2\phi_{r, k -1}^2 \\
	& \leq 4 \sigma_*^2 + 4k^2M_*\norm{\nabla f(\xx)}_2^2 \\
	& \leq 4\sigma_*^2 + 4n^2M_*\norm{\nabla f(\xx)}_2^2
\end{align*}
This means that Assumption~\ref{ass:heterogenity} should hold for some $\sigma^2 \leq 4\sigma_*^2$ which in turn means $\sigma_* \in \Omega(\sigma^2)$.

\section{Test Accuracy During Training With and Without Algorithm~\ref{alg:greedyalg}}
\label{app:fig_cifar10_single_class_batch}
\begin{figure}[H]
		\centering
		\includegraphics[width=.6\linewidth]{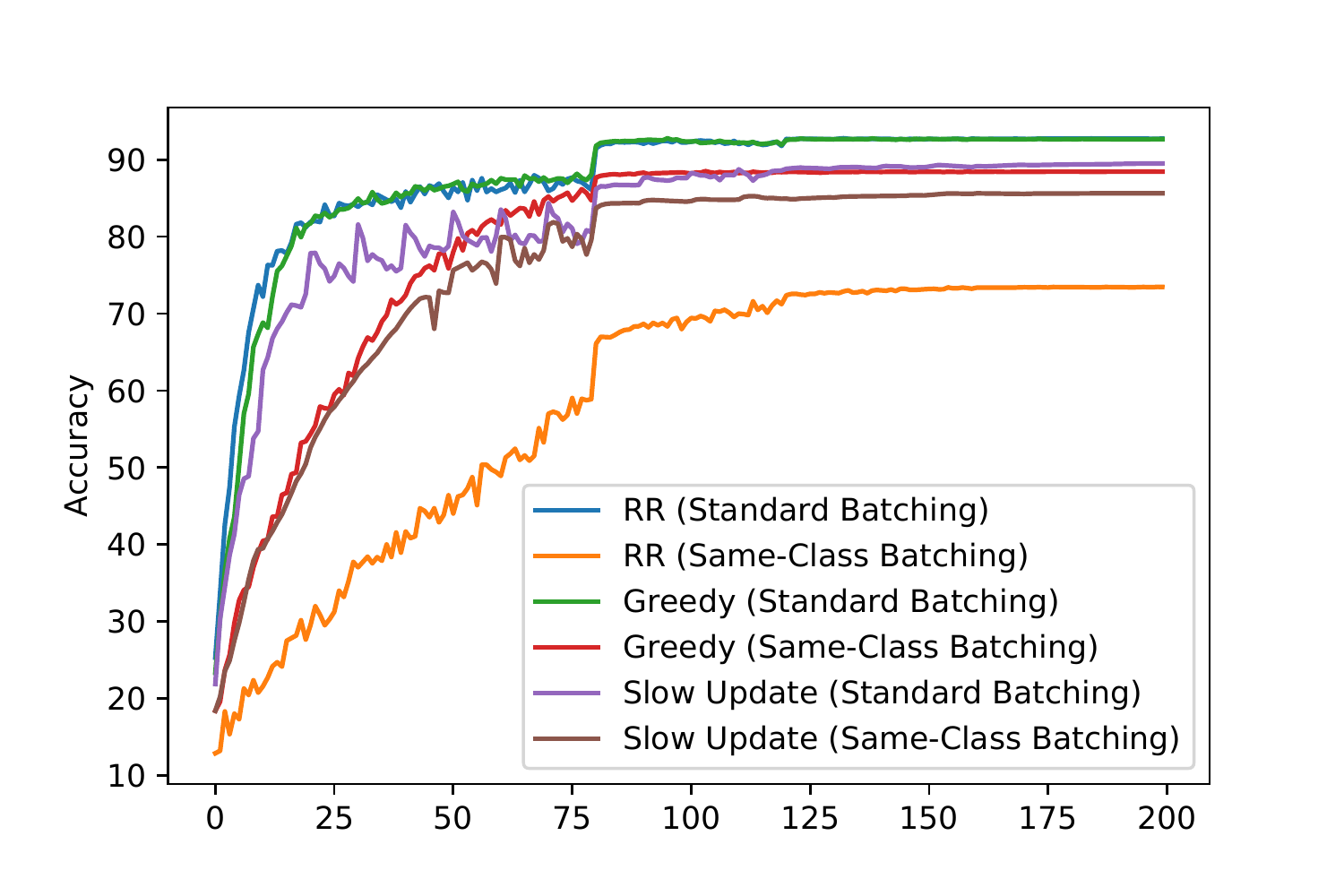}
		\caption{Test accuracy plot of using random-reshuffling and the greedy algorithm with same-class batching. Greedy refers to Algorithm~\ref{alg:greedyalg} while RR refers to Random Reshuffling. Algorithm~\ref{alg:greedyalg} clearly outperforms RR even when the order is only updated every 10 epochs (denoted by Slow Update).}
		\label{fig:cifar10_single_class_batch}
	\end{figure}

\end{document}